\documentclass[twoside,11pt]{article}

\usepackage{jmlr2e}
\usepackage{amscd,yhmath} 
\usepackage{wasysym}
\usepackage{hyperref}
\usepackage{mathrsfs}
\def\R{{\mathbb R}}

\def\E{{\mathbb E}}

\def\N{{\mathbb N}}
\def\s{{\mathbb S}}
\def\e{{\epsilon}}

\newcommand{\keyw}[1]{{\bf #1}}

\hypersetup{
  colorlinks=true,       
  linkcolor=magenta,          
   citecolor=blue,        
    filecolor=magenta,      
    urlcolor=cyan           
}

\jmlrheading{0}{0000}{0-00}{0/00}{00/00}{}{Vladimir Pestov}

\ShortHeadings{A learning rule with a universally monotone error}{Pestov}
\firstpageno{1}

\usepackage{lastpage}
\jmlrheading{23}{2022}{1-\pageref{LastPage}}{9/21}{4/22}{21-1043}{Vladimir Pestov}
\ShortHeadings{Universally monotone error}{Pestov}

\begin{document}

\title{A universally consistent learning rule with a universally monotone error}

\author{\name Vladimir Pestov \email
vladimir.pestov@uottawa.ca \\
       \addr Departamento de Matem\'atica \\
Universidade Federal de Santa Catarina \\
Campus Universit\'ario Trindade \\
CEP 88.040-900 Florian\'opolis-SC, Brasil \\
{and} \\
Departement of Mathematics and Statistics\\ 
University of Ottawa\\
STEM Complex, 150 Louis-Pasteur Pvt\\
Ottawa, Ontario K1N 6N5 Canada
}

\editor{Mehryar Mohri}

\maketitle

\begin{abstract}
We present a universally consistent learning rule whose expected error is monotone non-increasing with the sample size under every data distribution. The question of existence of such rules was brought up in 1996 by Devroye, Gy\"orfi and Lugosi (who called them ``smart''). Our rule is fully deterministic, a data-dependent partitioning rule constructed in an arbitrary domain (a standard Borel space) using a cyclic order. The central idea is to only partition at each step those cyclic intervals that exhibit a sufficient empirical diversity of labels, thus avoiding a region where the error function is convex.
\end{abstract}


\section{Introduction}

Here we are interested in the learning rules for a binary classification problem. Given a labelled $n$-sample $\sigma_n$, such a rule outputs a binary classifier for the domain $\Omega$, that is, predicts a label, $0$ or $1$, for every point $x$ of the domain. We denote the rule by ${g}=({g}_n)_{n=1}^{\infty}$, and the predicted label for $x$ based on a sample $\sigma_n$, by ${g}_n(\sigma_n)(x)$. Now let $\tilde\mu$ be an unknown distribution of the labelled datapoints $(X,Y)$, that is, a probability measure on $\Omega\times\{0,1\}$.
The learning (or generalization, or misclassification) error ${L}_{\tilde\mu}({g}_n)$ is the random variable 
\[P_{\tilde\mu}[{g}_n(D_n)(X)\neq Y\mid D_n],\]
where $D_n$ is a random labelled $n$-sample.
The rule ${g}$ is {\em consistent} (under $\tilde\mu$), if the error converges to the smallest possible classification error (the Bayes error), $ L^{\ast}(\tilde\mu)$, in expectation (or probability):
\[\E_{\tilde\mu}[{L}_{\tilde\mu}({g}_n)]\overset{n\to\infty}\longrightarrow L^{\ast}(\tilde\mu).\]
The rule is {\em universally consistent} if it is consistent under every data distribution $\tilde\mu$. 
Intuitively, this means that the more data we have, the better is the prediction of the learning rule, and asymptotically as $n\to\infty$, it is as good as it can possibly get under the (unknown) data law.

It is therefore tempting to think that the learning error does not increase under the transition $n\mapsto n+1$, that is, the sequence 
\begin{equation}
\E_{\tilde\mu}[{L}_{\tilde\mu}({g}_n)],~~n=1,2,3,\ldots
\label{eq:smart}
\end{equation}
is monotone nonincreasing. Perhaps surprisingly, it is not the case. See \citet{DGL}, Sect. 6.8 for a simple example of a data distribution on the interval, under which the nearest neighbour rule has a strictly smaller learning error for $n=1$ than it has for $n=2$. It is not difficult to construct similar counter-examples for other common universally consistent learning rules (cf. Problems 6.14 and 6.15, {\em loco citato}).

Devroye, Gy\"orfi and Lugosi called a rule ${g}$ {\em smart} (\citet{DGL}, Sect. 6.8) if for all labelled data distributions $\mu$ on $\Omega\times\{0,1\}$, the sequence in Eq. (\ref{eq:smart}) is nonincreasing.
Based on the above, they have conjectured that no universally consistent learning rule is ``smart''. (Cf. {\em loc. cit.}, bottom of p. 106 and Problem 6.16, p. 109.)

Our aim is to show that ``smart'' universally consistent rules do exist, even without requiring any amount of randomization. 

We use a partitioning rule: the domain is divided in disjoint cells, and the label for each cell is determined by the majority vote among all datapoints contained in it. It is easy to show that for a fixed partition, the error does not increase with the sample size (Problem 6.13 in \citet{DGL}). However, for a partitioning rule to be consistent, the cells have to be divided, and this is where the error jump may occur.

Here is the root of the problem. Let $Y,Y_i$, $i=1,\ldots,n$ be i.i.d. random labels following a Bernoulli distribution with $p=P[Y=1]$. Consider the predictor for the value of $Y$ based on the majority vote among $Y_1,\ldots,Y_n$. For the odd values of $n$, the voting ties are avoided, and  the misclassification error is a polynomial function in $p$:
\[L(p,n)=P[Y=1]P\left[\frac 1n\sum_{i=1}^n Y_i<\frac 12\right]+P[Y=0]P\left[\frac 1n\sum_{i=1}^n Y_i>\frac 12\right].
\]
For $n>1$, the error function is not concave: there is a straight line segment joining two points on the graph that is strictly above the underlying part of the graph (Fig. \ref{fig:non-concavity}).

\begin{figure}[ht]
\begin{center}
  \scalebox{0.4}{\includegraphics{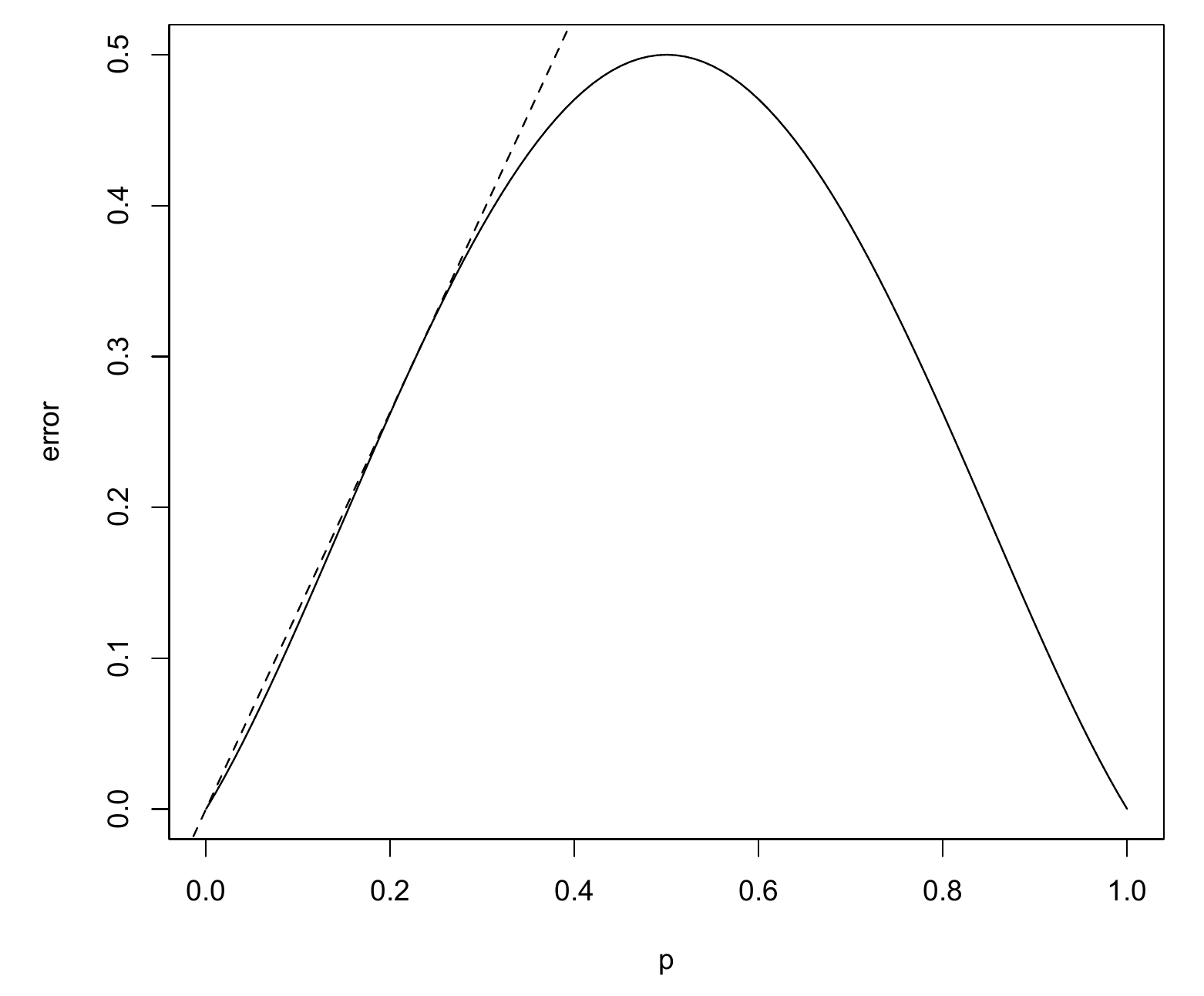}} 
  \caption{Non-concavity of the error function $L(p,3)$.}
  \label{fig:non-concavity}
\end{center}
\end{figure}

This is because, for $n>1$, the derivative of the polynomial function ${L}(p,n)$ at $p=0$ and $1$ equals $1$ and $-1$ respectively (see Problem 5.6(2) in \citet{DGL}, p. 84 for a Taylor polynomial). The optimal (Bayes) predictor for the problem gives the value $0$ if $p<1/2$ and $1$ if $p>1/2$, and the Bayes error is given by $L^{\ast}(p) = \min\{p,1-p\}$.
Since ${L}(p,n)>  \min\{p,1-p\}$ at all points except $0,1/2,1$, it follows that there are small neighbourhoods of $0$ and $1$ in which the polynomial function ${L}(p,n)$ is strictly convex. 

This implies that no concave function strictly greater than $\min\{p,1-p\}$ is contained under the graph of $L(p,n)$. 
In particular, given $N>n$, for some $p_0>0$ small enough
\begin{equation}
{L}(p_0,n)< \frac 12 {L}(2p_0,N).
\label{eq:p00}
\end{equation}

And here is how the error value can increase after we refine the partition, even if we increase the sample size. Suppose the domain $\Omega$ is subdivided into two cells of equal measure, $C_1$ and $C_2$, with conditional probabilities of getting label $1$ equal to $2p_0$ and $0$ respectively. (Fig. \ref{fig:splitting}.) 

\begin{figure}[ht]
\begin{center}
  \scalebox{0.3}{\includegraphics{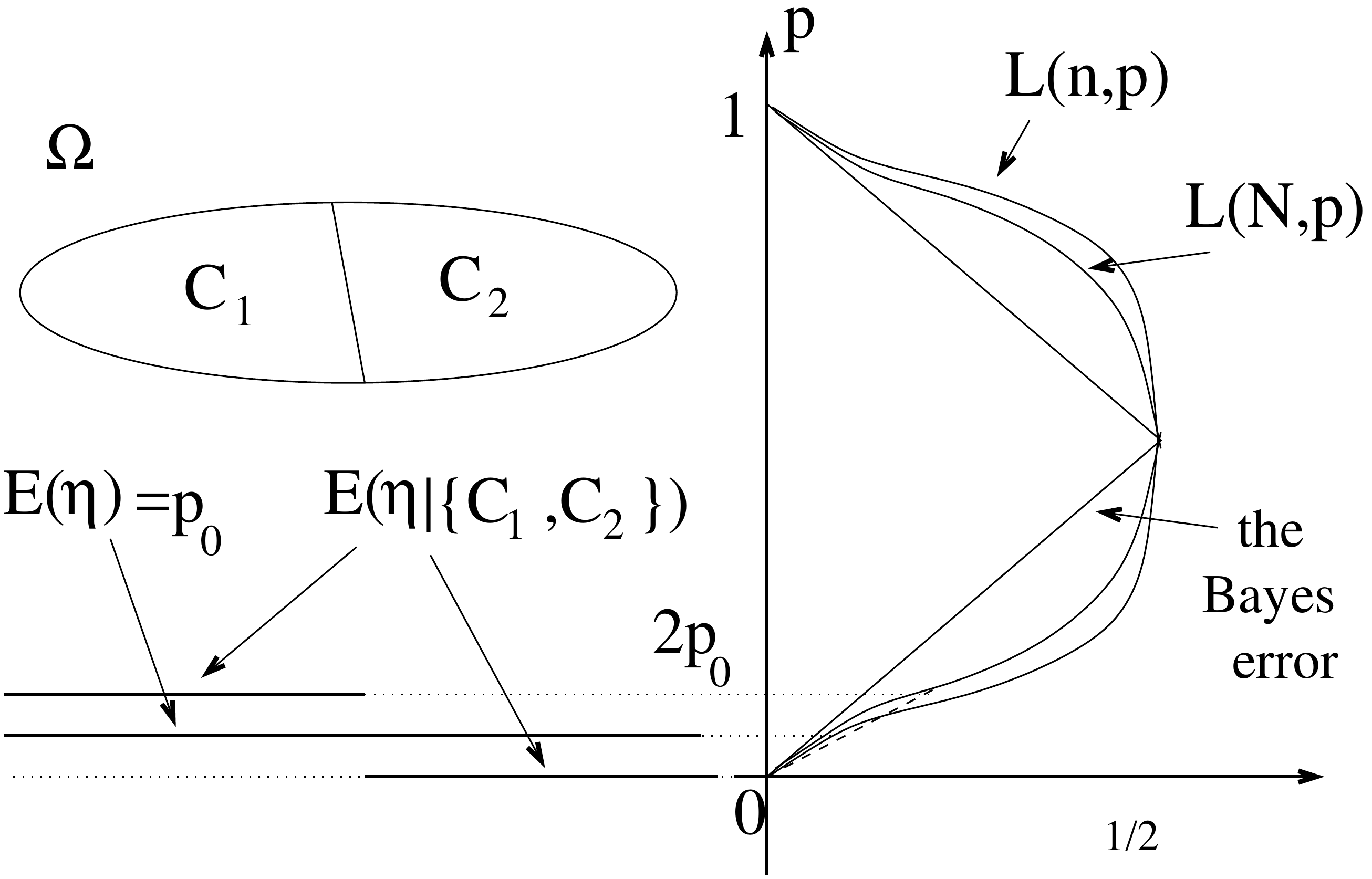}} 
  \caption{Splitting the domain into two cells in the area of convexity of error.}
  \label{fig:splitting}
\end{center}
\end{figure}

The learning error of the rule based on the trivial partition, $\{\Omega\}$, and a random $n$-sample $\sigma$ equals $L(p_0,n)$. But when we proceed to the rule based on the finer partition $\{C_1,C_2\}$, the error, conditionally on each cell containing $N$ sample points, strictly increases:
\begin{align*}
P[X\in C_1]\cdot L(2p_0,N)+ P[X\in C_2]\cdot L(0,N)&= \frac 12 {L}(2p_0,N)+\frac 12 0 \\
&>L(p_0,n).
\end{align*}
Using the monotonicity of $L(p,n)$ in $n$, it is easy to deduce that the expected error of the histogram rule based on the trivial partition, $p_{\{\Omega\}}$, over i.i.d. $n$-samples is less than the expected error of the rule based on the finer partition $\{C_1.C_2\}$, over the i.i.d. $N$-samples.

However, away from the endpoints of the interval $[0,1]$ this phenomenon no longer occurs. Given $\e>0$, if $N$ is sufficiently large, then on the interval $[\e,1-\e]$ the concave envelope of the function $L(p,N)$ (that is, the smallest concave function majorising it) is smaller than the function $L(p,n)$. This means that a cell $C$ can be safely partitioned (into any finite number of smaller cells) once the conditional probability $p=P[Y=1\mid X\in C]$ is bounded away from $0$ and $1$, that is, belongs to some interval $[\e_n,1-\e_n]$, where  $\e_n\downarrow 0$.
Therefore, the solution is to only partition a cell $C$ when it is empirically confirmed that $P[Y=1\mid X\in C]$ is in the interval $[\e_n,1-\e_n]$:
\begin{align*}
P_{\sigma}[Y=1\mid X\in C] &= \frac{\sum_{i\colon X_i\in C}Y_i}{\sharp\{i\colon X_i\in C\}} \\
&\in [\e_n,1-\e_n].
\end{align*}
Here $P_\sigma$ stands for the empirical (conditional) probability based on a random labelled sample $\sigma$.

There is still a probability of empirical error, but, 
near $0$ and $1$ and for $\e_n$ fixed, this error is a polynomial function in $p$ (resp. $1-p$) of higher order $\e_nN$, where $N$ is the number of points of the testing sample contained in the cell. Thus, even if we may from time to time erroneously partition a cell when we should not, the expected compound error under the transition $n\mapsto N$ still can be kept below the curve of the error function $L(p,n)$, provided $N$ is large enough. (Lemmas \ref{l:key0} and \ref{l:key_piece}.)

Now, a description of the learning rule, $g=(g_n)$.
The empirical path $(x_n)$ is divided into points of three kinds. A subsequence $(n_k)$ is chosen, starting with $n_1=1$. 
The points $(x_{n_k+1})$ (with labels stripped off) are used to form a partition of the domain into half-open subintervals, for which purpose we fix a circular order on the domain. Equivalently, we identify the domain with the unit circle $\s^1$. This way, almost surely the measure of every cell of the partition is strictly positive. (Fig. \ref{fig:cyclic_intervals}.)

\begin{figure}[ht]
\begin{center}
  \scalebox{0.25}{\includegraphics{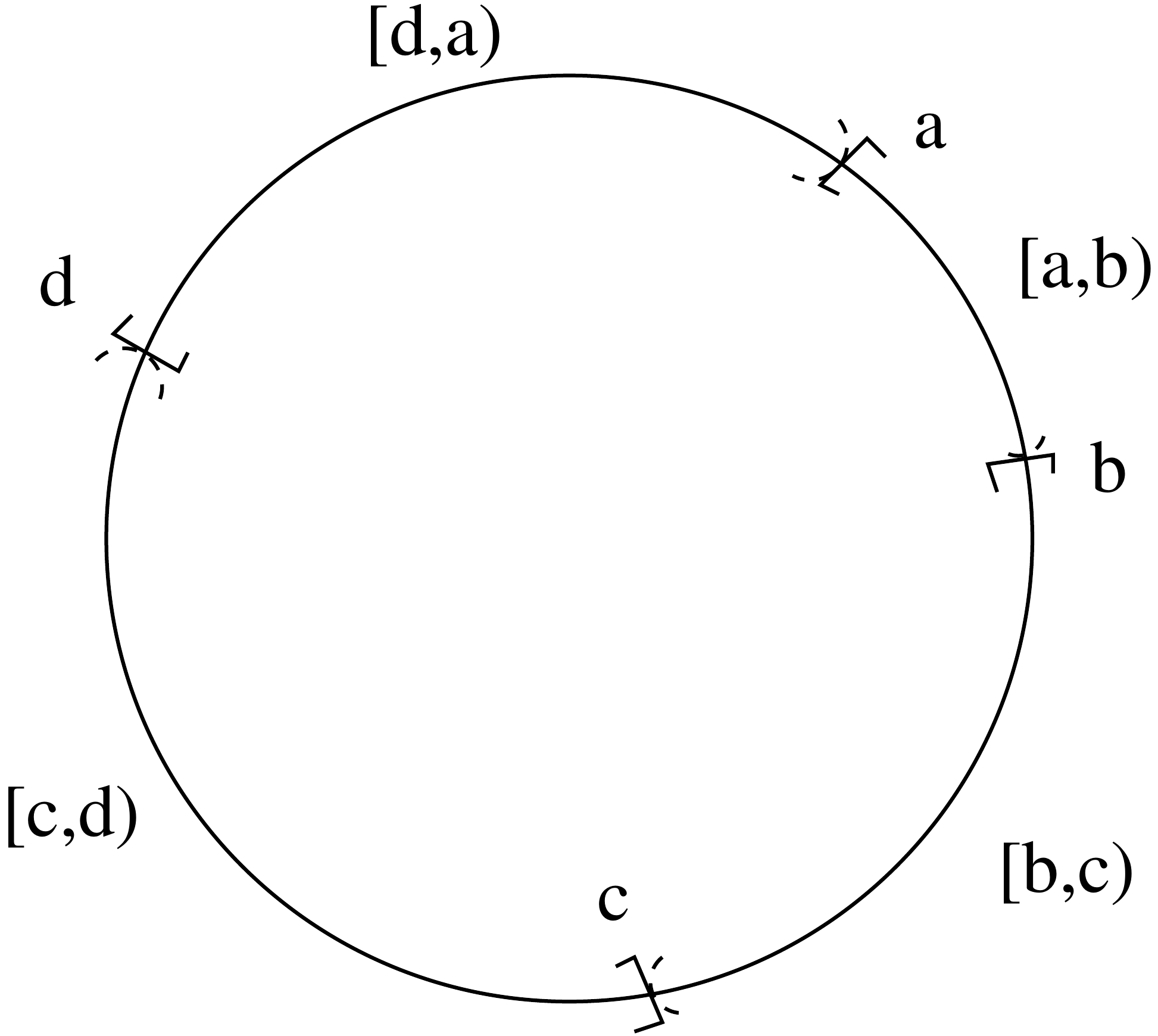}} 
  \caption{Half-open cyclic intervals.}
  \label{fig:cyclic_intervals}
\end{center}
\end{figure}

The hypothesis is only updated at the steps of the form $n=n_{k}$, so for all the intermediate values $n_k<n<n_{k+1}$, the rule $g_n$ just repeats the hypothesis output by the rule $g_{n_k}$: 
\[g_n(\sigma)=g_{n_k}(\sigma[1,\ldots,n_k]).\]
The hypotheses $g_{n_k}$ are generated recursively, that is, in order to output the hypothesis $g_{n_k}$, we need to know the hypotheses $g_{n_i}$, $i=1,2,\ldots,k-1$. In particular, the partitioning set ${\mathcal Q}_k\subseteq \{x_{n_1},\ldots,x_{n_k}\}$ is selected recursively as well.

The interval of integers $[n_k+2,n_{k+1}]$ is divided into two contiguous blocks, $A_{k}$ and $B_{k}$, of length $a_{k}$ and $b_{k}$ respectively. Thus, $n_{k+1}=n_k+1+a_{k}+b_{k}$.
We call $A_{k}$ the testing block, and $B_{k}$, the labelling block. (See Fig. \ref{fig:three_kinds}.)

\begin{figure}[ht]
\begin{center}
  \scalebox{0.3}{\includegraphics{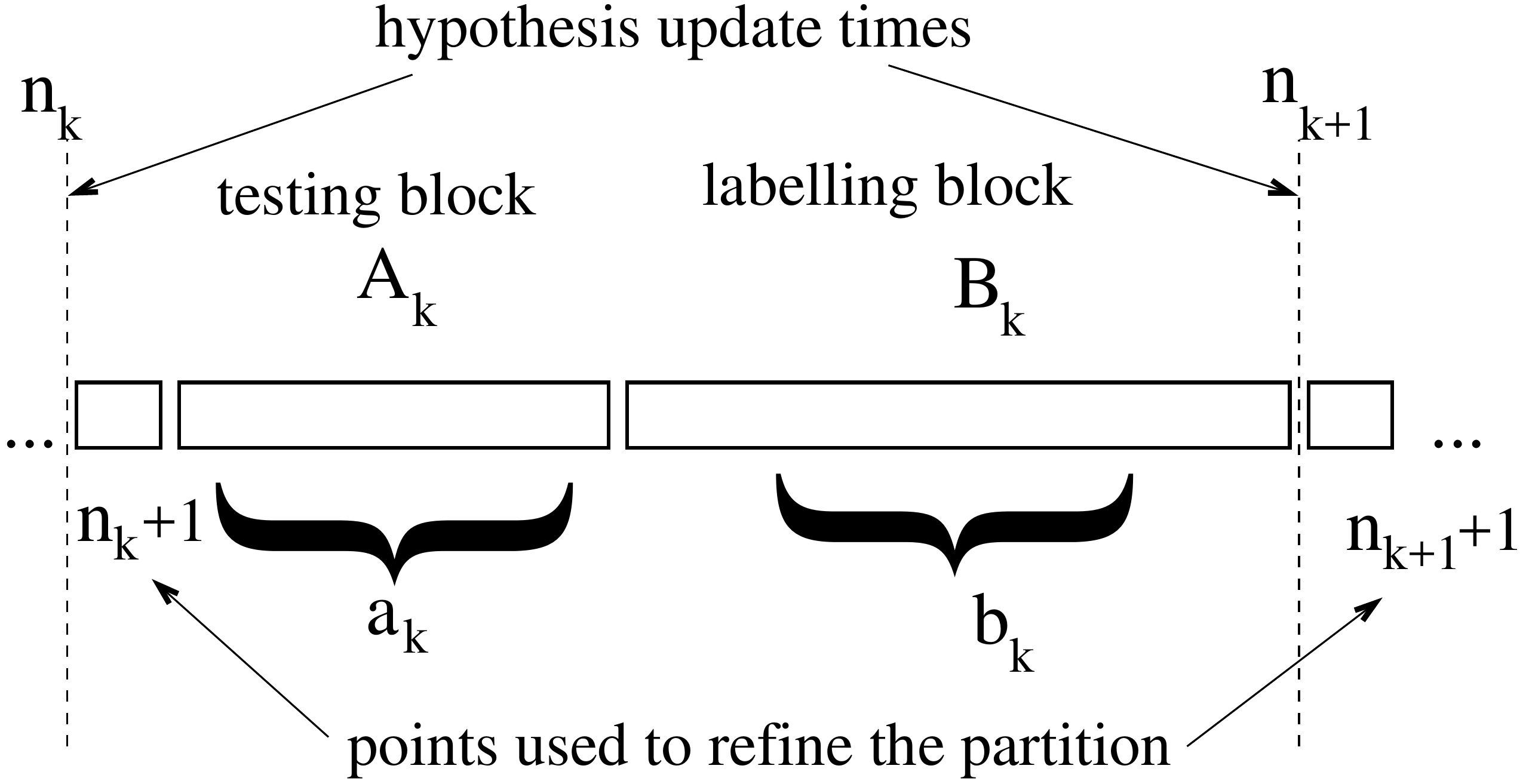}} 
  \caption{Natural numbers divided into blocks.}
  \label{fig:three_kinds}
\end{center}
\end{figure}

The testing block, $A_{k}$, or rather the corresponding subsample $\sigma[A_{k}]$, is used for empirical testing of the probability value $P[Y=1\mid X\in I]$ for every cell $I$ (a half-open cyclic interval) of the  partition, $\hat {\mathcal Q}_{k-1}$, generated by the current partitioning set ${\mathcal Q}={\mathcal Q}_{k-1}$. And the labelling block, $B_{k}$, is used to generate the predicted labels based on the partitioning rule under the possibly updated partition.

We only perform the testing and then generate a new hypothesis if every cell of the partition $\hat {\mathcal Q}_{k-1}$, contains sufficiently many points of the testing sample, and every cell of the partition $\hat{\mathcal P}_k$ contains sufficiently many points of the labelling sample. If this is the case, then for every interval $I\in\hat {\mathcal Q}_{k-1}$ satisfying $P_{\sigma[A_k]}[Y=1\mid X\in I]\in (\e_k,1-\e_k)$ we update the partitioning set ${\mathcal Q}_{k-1}$ by adding to it the set ${\mathcal P}_k\cap I$. After the update is done, we generate the new hypothesis, given by the histogram rule on the updated partition $\hat {\mathcal Q}_k$ and the sample $\sigma[B_k]$. Otherwise, no testing and no partitioning set update are made, and the rule returns the previous hypothesis output at the moment $n=n_{k-1}$.

At the beginning, the set of partitioning points is initialized to be an empty set, ${\mathcal Q}_0=\emptyset$, hence the corresponding partition $\hat {\mathcal Q}_0$ is trivial, $\hat {\mathcal Q}_0=\{\s^1\}$, having the entire domain $\Omega=\s^1$ as the only cell. Thus,
for $n=1$, the rule $g_1$ on the labelled sample input $(x_1,y_1)$ assigns the label $y_1$ to every point of the domain. In this way, $(x_1,y_1)$ is the subsample $\sigma[B_1]$, the first labelling block is $B_1=\{1\}$, and $b_1=1$.
The first testing block is empty: $A_1=\emptyset$, and $a_1=0$.

Further values $a_k,b_k$ will be chosen recursively so as to grow fast enough and guarantee that almost surely along the sample path, the hypothesis is being updated infinitely often. A conditioning argument shows the misclassification error does not increase with each step. As to the universal consistency, notice that in the cases where the regression function is identically zero or one (that is, the same label is assigned to every point), the division of cells will almost surely never occur, but obviously the partitioning rule with only one cell is still consistent. However, a certain variation of existing results on the universal consistency of partitioning rules (which require the cell diameter to go to zero in probability) does the job.

Now an outline of the article structure. The sections \ref{s:rules}--\ref{s:cyclic} lay the technical groundwork for our learning rule, presented and studied in the sections \ref{s:rule}--\ref{s:consistency}.
We start with a revision of the standard model of supervised statistical learning in Sect. \ref{s:rules}. The concavity properties of the error function $L(p,n)$ are dealt with in Section \ref{s:trivial}. In Sect. \ref{s:partitioning} we discuss the partitioning rules, and here appears the key technical result, Lemma \ref{l:key_piece}, showing how to partition cells without increasing the learning error.
Section \ref{s:cyclic} is devoted to the cyclic orders and the probability measures on the circle. Our learning rule is formally described in Section \ref{s:rule}. In Sect. \ref{s:monotonicity} we show the rule has a universally monotone expected error, and in Sect. \ref{s:consistency} we prove the universal consistency. A few concluding remarks in Sect. \ref{s:remarks} were motivated by the referees' comments.

\section{Learning rules}
\label{s:rules}
Let $\Omega=(\Omega,{\mathcal A})$ be a measurable space, that is, a set equipped with a sigma-algebra of subsets $\mathcal A$. The main example is where $\mathcal A$ is the smallest sigma-algebra of subsets containing all open balls with regard to a metric, making $\Omega$ a complete separable metric space. Such a measurable space $(\Omega,{\mathcal A})$ is called a {\em standard Borel space}. The elements of $\mathcal A$ are known as {\em Borel sets.} 

The Borel structure remembers very little of the generating metric, in the following sense.
Two standard Borel spaces admit a Borel isomorphism (a bijection preserving the sigma-algebras) if and only if they have the same cardinality. Thus, there are only the following isomorphism types of standard Borel spaces: finite ones with $n$ elements for each natural $n$, countably infinite (all of which are isomorphic to the natural numbers with the sigma-algebra of all subsets), and those of cardinality continuum. For example, the Borel spaces associated to the real line, to $\R^d$, to the Hilbert space $\ell^2$, to the Cantor set, etc., are all pairwise isomorphic as standard Borel spaces. See \citet{kechris} as a general reference.

In statistical learning theory the learning domain $\Omega$ is usually assumed to be a standard Borel space of cardinality continuum, which will be our standing assumption as well. 

The product $\Omega\times \{0,1\}$ now becomes a standard Borel measurable space in a natural way. The elements $x\in \Omega$ are known as {\em unlabelled points}, and the pairs $(x,y)\in \Omega\times \{0,1\}$ are {\em labelled points.} A finite sequence of labelled points, $\sigma=(x_1,x_2,\ldots,\allowbreak x_n,\allowbreak y_1,\allowbreak y_2,\ldots,y_n)\in \Omega^n\times \{0,1\}^n$, is a {\em labelled sample.}

A {\em classifier} in $\Omega$ is a mapping
\[T\colon \Omega\to \{0,1\},\]
assigning a label to every point. The mapping is usually assumed to be measurable in order for things like the misclassification error to be well defined, although some authors are allowing for non-measurable maps, working with the outer measure instead.

Let $\tilde\mu$ be a probability measure defined on the measurable space $\Omega\times \{0,1\}$. Denote $(X,Y)$ a random element of $\Omega\times \{0,1\}$ following the law $\tilde\mu$. The misclassification error of a classifier $T$ is the quantity
\begin{align*}
{L}_{\tilde\mu}(T) &= \tilde\mu \{(x,y)\in \Omega\times \{0,1\}\colon T(x)\neq y\} \\
&= P[T(X)\neq Y].
\end{align*}
The {\em Bayes error} is the infimum (in fact, the minimum) of the misclassification errors of all the classifiers $T$ defined on $\Omega$:
\[ L^{\ast}= L^{\ast}(\tilde\mu)=\inf_{T}{{L}}_{\tilde\mu}(T).\]

A {\em learning rule} in $(\Omega,{\mathcal A})$ is a mapping
\[{g}\colon\bigcup_{n=1}^{\infty}\Omega^n\times \{0,1\}^n\times \Omega \ni (\sigma,x)\mapsto {g}(\sigma)(x) \in \{0,1\}.\]
Again, the map above is usually assumed to be measurable. We denote $g_n$ the restriction of $g$ to $\Omega^n\times \{0,1\}^n\times \Omega$. For a labelled sample $\sigma$, we denote $g_n(\sigma)$ the binary function $\Omega\ni x\mapsto g_n(\sigma)(x)\in\{0,1\}$. Thought of as a subset of $\Omega$ on which the function takes value $1$, this $g_n(\sigma)$ is also known as a {\em hypothesis} output by the rule $g$ on the labelled sample input $\sigma$.

The labelled datapoints are modelled by a sequence of independent, identically distributed random elements $(X_n,Y_n)\in \Omega\times \{0,1\}$ following the law $\tilde\mu$. For each $n$, the {\em misclassification error} of the rule ${g}_n$ is the random variable
\begin{align*}
{L}_{\tilde\mu}{g}_n & = P\left[ {g}_n(D_n)(X)\neq Y\mid D_n\right].
\end{align*}
In other words, it is the error of the random classifier $g_n(D_n)$, where $D_n$ is a random labelled $n$-sample.

Consider the probability measure $\mu = \tilde\mu\circ\pi^{-1}$ on $\Omega$, where $\pi$ is the first coordinate projection of $\Omega\times \{0,1\}$. Now define a finite measure $\mu_1$ on $\Omega$ by $\mu_1(A)=\tilde\mu(A\times\{1\})$. Clearly, $\mu_1$ is absolutely continuous with regard to $\mu$. Define the {\em regression function}, $\eta\colon \Omega\to [0,1]$, as the Radon--Nikod\'ym derivative
\begin{align*}
\eta(x) & =\frac{d\mu_1}{d\mu} \\
&= P[Y=1\mid X=x],
\end{align*}
that is, the conditional probability for $x$ to be labelled $1$. The pair $(\mu,\eta)$ completely defines the measure $\tilde\mu$ and is often more convenient to use. Thus, a learning problem in a measurable space $(\Omega,{\mathcal A})$ can be alternatively given either by the measure $\tilde\mu$ on $\Omega\times\{0,1\}$ or by the pair $(\mu,\eta)$.

A rule ${g}$ is {\em consistent} under $\tilde\mu$ if
\[{L}_{\tilde \mu}{g}_n \overset{p}\to  L^{\ast}(\tilde\mu),\mbox{ that is, }
\E {L}_{\tilde \mu}{g}_n\to L^{\ast}(\tilde\mu),\]
and {\em universally consistent} if ${g}$ is consistent under every probability measure on $\Omega\times \{0,1\}$.

The following simple observation allows to delay the hypothesis update.

\begin{lemma} 
Let $F$ be a random function from the natural numbers to itself with the property $F(n)\leq n$ for all $n$. Given a learning rule $g$, define a rule $g_F$ by $(g_F)_n(\sigma) = g_{F(n)}(\sigma[(1,\ldots,F(n))])$. 
Assume that $F(n)-n\to 0$ in probability. If $g$ is consistent, then $g_F$ is consistent.
\label{l:random_delay}
\end{lemma}

\begin{proof}
Given $\e>0$, find $N$ so that for $n\geq N$, we have $P(L_n(g_n)<L^\ast+\e)>1-\e$ and $P[F(n)=n]>1-\e$ whenever $n\geq N$. It follows that for such $n$
\begin{align*}
P(L(g_F)_n< L^\ast+\e)\geq P(F(n)&=n \mbox{ and }L_n(g_n)<L^\ast+\e)\\
&>1-2\e.
\end{align*}
\end{proof}

\section{Error in a trivial one-point domain}
\label{s:trivial}

Consider the ``natural'' learning rule $g$ in the one-point domain $\Omega=\{\ast\}$, which is the majority vote among $n$ i.i.d. random labels $D_n=(Y_1,\ldots,Y_n)$ following the Bernoulli law. To avoid ties, we assume $n$ odd (so if $n$ is even, the $n$-th label is not considered).

So, let $n=2k+1$, $k\geq 0$ and $p\in[0,1]$. We have:
\[P[{g}_n(D_n)=0]= P\left[\frac 1n\sum_{i=1}^nY_i<1/2\right] =
\sum_{i=0}^k {n\choose i}p^i(1-p)^{n-i},\]
\[
P[{g}_n(D_n)=1]=\sum_{i=0}^k {n\choose i}p^{n-i}(1-p)^{i}.
\]
The expression for the expected learning error becomes
 \begin{align}
 {L}\,(p,n) &= P[{g}_n(D_n)=0]\cdot P[Y=1] + P[{g}_n(D_n)=1]\cdot P[Y=0] \nonumber \\
 &= \sum_{i=0}^k {n\choose i} p^{i+1}(1-p)^{n-i}+ \sum_{i=0}^k {n\choose i} p^{n-i}(1-p)^{i+1}.
\label{l:learningerror}
 \end{align}

As $n\to\infty$, simple ball volume considerations in the Hamming cube $\{p,1-p\}^n$ show that $L(p,n)$ converges to the Bayes error,
\[ L^{\ast}(p) = \min\{p,1-p\}.\]

For the first part of the following statement, see Sect. VII of \citet{CH}, or Problem 6.12 in \citet{DGL}; as neither source contains a proof, we present it here.

\begin{lemma}
The convergence ${L}\,(p,n)\to  L^{\ast}(p)$ is monotone along the odd values of $n$. More exactly, for every $n=2k+1$,
\begin{equation}
{L}\,(p,n)-{L}\,(p,n+2) = {n\choose k}(2p-1)^2p^{k+1}(1-p)^{k+1} \geq 0.
\end{equation}
\label{l:monotone}
\end{lemma}

\begin{proof}
Let $\{p,1-p\}^n$ denote the Hamming cube $\{0,1\}^n$ with the product measure $\{p,1-p\}^{\otimes n}$. Denote $w(\tau)=\sum_{i=1}^n\tau_i$ the weight of the string $\tau$.
The learning error ${L}\,(p,2k+1)$ is the expected value of the random variable $L(g_n(D_n))$, represented by the function
\[\{p,1-p\}^n\ni\tau\mapsto L(g_n(\tau))=\begin{cases} p,&\mbox{ if }w(\tau)\leq k,\\
1-p,&\mbox{ if }w(\tau)\geq k+1.
\end{cases} \]
We will study the behaviour of this expected value under the transition $k\mapsto k+1$, that is, $n\mapsto n+2$. In the latter case, 
\[\tau\mapsto L(g_{n+2}(\tau))=\begin{cases} p,&\mbox{ if }w(\tau)\leq k+1,\\
1-p,&\mbox{ if }w(\tau)\geq k+2.
\end{cases} \]

For every $i=0,1,\ldots,n$, denote $C_i$ the cylindrical set of all strings $\tau\in\{0,1\}^{n+2}$ satisfying $w(\tau_n)=i$, where $\tau_n$ is the $n$-prefix of $\tau$. We have
\begin{align*}
{L}\,(p,n)-{L}\,(p,n+2) &= \E\, L(g_n(\tau_n))-\E L(g_{n+2}(\tau)) \\
&= \sum_{i=0}^n \int_{C_i} \left(L(g_n(\tau_n))-L(g_{n+2}(\tau))\right)\,d\{p,1-p\}^{\otimes n}.
\end{align*}

If the prefix $\tau_{n}$ of $\tau$ satisfies $w(\tau_n)\leq k-1$, then $w(\tau)\leq k+1$, and so on $C_i$ the two error variables  are identical. The same applies if $w(\tau_n)\geq k+2$. We have
\[{L}\,(p,n+2)-{L}\,(p,n)=\sum_{i=k,k+1} \int_{C_i} \left(L(g_n(\tau_n))-L(g_{n+2}(\tau))\right)\,d\{p,1-p\}^{\otimes n}.\]
{\em Case $i=k$.} The value of the error variable changes from $p$ to $1-p$ for the strings of the form $\tau=\tau_n11$, whose total number is ${n\choose k}$. For every such string, the singleton $\{\tau_n 11\}$ has measure $p^{k+2}(1-p)^{k+1}$. Other strings in $C_k$ keep the same error value, $p$. We conclude:
\[\int_{C_{k}}\left(L(g_n(\tau_n))-L(g_{n+2}(\tau))\right)\,d\{p,1-p\}^{\otimes n}=(2p-1){n\choose k}p^{k+2}(1-p)^{k+1}.
\]
{\em Case $i=k+1$.} The error value changes from $1-p$ to $p$ for the strings of the form $\tau_n00$, whose total number is ${n\choose k+1}={n\choose k}$:
\[\int_{C_{k+1}}\left(L(g_n(\tau_n))-L(g_{n+2}(\tau))\right)\,d\{p,1-p\}^{\otimes n}=(1-2p){n\choose k}p^{k+1}(1-p)^{k+2}.
\]
Thus, 
\begin{align*}
{L}\,(p,n)-{L}\,(p,n+2)&= (2p-1){n\choose k}p^{k+2}(1-p)^{k+1}+(1-2p){n\choose k}p^{k+1}(1-p)^{k+2}\\
&= (2p-1){n\choose k}p^{k+1}(1-p)^{k+1} \left[p - (1-p)  \right] \\
&= (2p-1)^2{n\choose k}p^{k+1}(1-p)^{k+1} \geq 0.
\end{align*}
\end{proof}

\begin{lemma}
The convergence ${L}\,(p,n)\to  L^{\ast}(p)$ is uniform in $p$ along the odd values of $n$.
\label{l:uniformL}
\end{lemma}

\begin{proof}
The sequence of non-negative continuous functions ${L}\,(p,n)- L^{\ast}(p)$ with $n$ odd converges to zero pointwise and monotonically on a compact set $[0,1]$, which implies the uniform convergence (Dini's theorem).
\end{proof}

A real function $f$ is {\em concave} if for all $x,y$ in the domain of $f$ and each $t\in [0,1]$, $f((1-t)x+ty)\geq (1-t)f(x)+tf(y)$.
Equivalently, for any collection of $x_i$ and $t_i$ with $\sum_i t_i=1$, we have $f\left(\sum_it_ix_i\right)\geq \sum_it_i f(x_i)$. 

\begin{lemma}
For every $n$ odd, the function ${L}\,(p,n)$ is concave in a sufficiently small neighbourhood of $p=1/2$.
\label{l:concave1/2}
\end{lemma}

\begin{proof}
For $n=1$, ${L}\,(p,1)=2p(1-p)=2(p-p^2)$ is globally concave. Write
\[p(1-p) = \frac 14 - \left(p-\frac 12\right)^2.\]
Then, by Lemma \ref{l:monotone},
\begin{align*}
{L}\,(p,n+2) &= {L}\,(p,n)-{n\choose k}(2p-1)^2p^{k+1}(1-p)^{k+1}\\
&= {L}\,(p,n)-4{n\choose k}\left(p-\frac 12\right)^2\left[ \frac 14 - \left(p-\frac 12\right)^2\right]^{k+1}.
\end{align*}
The lowest term in the Taylor expansion of the second polynomial around $p=1/2$ is of second degree, $(p-1/2)^2$ with negative coefficient $-4{n\choose k}(1/4)^{k+1}$, meaning the function is concave in a sufficiently small neighbourhood of $p=1/2$. As the sum of concave functions is concave, we conclude by induction.
\end{proof}

If $f$ is a bounded real-valued function defined on some set $X$, it is easy to see that there exists the smallest concave function of the same domain of definition majorizing $f$. This function is called the {\em concave envelope} of $f$, and we will denote it $\wideparen{f}$.

\begin{lemma}
Given $n$ and $\e>0$, there is $N$ so that for all values $\e\leq p\leq 1-\e$
\[\wideparen{{L}}\,(p,N)\leq {L}\,(p,n).\]
\label{l:wideparenL<L}
\end{lemma}

\begin{proof}
Choose $\delta>0$ so that in the $\delta$-neighbourhood of $p=1/2$, the function ${L}\,(p,n)$ is concave (Lemma \ref{l:concave1/2}). Since the only values where ${L}\,(p,n)=p$ are $p=0,1/2,1$, there is $\gamma>0$ so small that for all $p\in [\e,1/2-\delta]$ we have ${L}\,(p,n)>p+\gamma$. By the intermediate value theorem, there is $q\in [1/2-\delta,1/2)$ with ${L}\,(q,n)=q+\gamma$. Reducing $\gamma$ further if needed, we may assume that such a $q$ is unique. 
By the uniform convergence of error functions (Lemma \ref{l:uniformL}), there is $N$ with ${L}\,(p,N)\leq p+\gamma$ for all $p$. The monotonicity of the error function (Lemma \ref{l:monotone}) implies ${L}\,(p,N)\leq {L}\,(p,n)$ for all $p$. 

The function
\[\psi(p)=\begin{cases} p+\gamma,&\mbox{ if }p\leq q, \\
{L}\,(p,n),&\mbox{ if } q\leq p \leq 1/2,
\end{cases} \]
extended by symmetry over $[1/2,1]$, is concave over $[0,1]$. (Indeed, for every $p\in [q,1/2]$, the gradient of the chord joining $(p,L(p,n))$ with $(q,q+\e)=(q,L(q,n))$ is less than $1$.) By the construction, we have ${L}\,(p,N)\leq \psi(p)$. Therefore, for all $p$,
\[\wideparen{{L}}\,(p,N)\leq \psi(p).\]
Since on the interval $[\e,1-\e]$ we have $\psi(p)\leq {L}\,(p,n)$, we conclude.
\end{proof}

The proof of the following is left out as an exercise in \citet{DGL}, problem 5.6(2). 

\begin{lemma}
Let $n=2k+1$, where $k\geq 1$. Up to higher degree terms, $L(p,n)$ at zero has the form
\[L(p,n)=p+{2k+1\choose k} p^{k+1}+o(p^{k+1}).\] 
\label{l:expansionat0}
\end{lemma}

\begin{proof}
Consider the expression for the learning error (Eq. \ref{l:learningerror}):
\[L(p,n)=
\sum_{i=0}^k {n\choose i} p^{i+1}(1-p)^{n-i}+ \sum_{i=0}^k {n\choose i} p^{n-i}(1-p)^{i+1}.\]
The monomial of the lowest order in the right hand sum comes from the term corresponding to $i=k$ and equals exactly ${2k+1\choose k} p^{k+1}$. The monomial of the lowest order in the left hand sum corresponds to $i=0$ and equals $p$. Thus, it is enough to show that in the polynomial
\[\sum_{i=0}^k {n\choose i} p^{i}(1-p)^{n-i}=
\sum_{i=0}^k {n\choose i} p^{i}\sum_{j=0}^{n-i}{n-i\choose j}(-1)^jp^j\]
(the l.h.s. after we took $p$ out) all the powers of $p$ between $m=1$ and $m=k$ inclusive vanish. Let $1\leq m\leq k$. Using the classical binomial formula, we calculate the coefficient of $p^m$:
\begin{align*}
\sum_{i+j=m} {n\choose i} {n-i\choose j}(-1)^j &=\frac{n!}{m! (n-m)!} \sum_{j=0}^m \frac{m!(-1)^j}{j! (m-j)!} \\
&= \frac{n!}{m! (n-m)!} (1-1)^m \\
&=0.
\end{align*}
\end{proof}

\begin{remark}
Note that 
\[{2k+1\choose k}={2k+1 \choose k+1} = {2k \choose k} + {2k \choose k+1},\]
which is how the expression for the coefficient appears in \citet{DGL}. Also, Lemma \ref{l:expansionat0} is false for $k=0$ (that is, $n=1$), in which case $L(p,1)=2p-2p^2$. There are two reasons why the proof fails: first, $p^{k+1}=p$, and second, we cannot conclude that for $m=k=0$ the power $(1-1)^m$ vanishes.
\end{remark}

The following key technical result together with its corollary underpins our learning rule by saying that a certain amount of empirical error when testing a cell for partitioning is admissible. An application to random partitions appears in Lemma \ref{l:key_piece}.

\begin{lemma}
Given $n$ odd and $t\in (0,1]$, for all $N=N(n,t)$ (odd) large enough,
\[P[\mathrm{binomial}(p,N)>tN]\cdot \wideparen{L}(p,N)+P[\mathrm{binomial}(p,N)\leq tN]\cdot L(p,N) \leq L(p,n)\]
over all $p\in [0,1/2]$.
\label{l:key0}
\end{lemma}

\begin{proof}
Let $n=2k+1$. By force of Lemma \ref{l:expansionat0},
\[\lim_{p\to 0}\frac{L(p,n+2)-p}{L(p,n)-p}=0,\]
and for some $\delta>0$ small enough,
\[L(p,N)-p \leq L(p,n+2)-p < \frac 12 (L(p,n)-p)\]
when $p\in [0,\delta]$ and $N>n$ is odd. Rewrite the inequality as
\[L(p,N) < \frac 12 (L(p,n)+p).\]

Now note a very rough estimate 
\begin{align*}
P[\mathrm{binomial}(p,N)\geq tN] &= \sum_{i=\lceil tN\rceil}^{N}{N\choose i} p^i (1-p)^{N-i} \\
&\leq p^{\lceil tN\rceil}\sum_{i=0}^N {N\choose i} \\
&\leq p^{tN}2^N \\
& = (2^{t^{-1}}p)^{tN}.
\end{align*}
When $N>t^{-1}(k+1)$, thanks to Lemma \ref{l:expansionat0}, the ratio of the polynomials $P[\mathrm{binomial}(p,N)>tN]$ and $(1/2)L(p,n)-p/2$ converges to zero as $p\to 0$, and so for some $\delta^\prime>0$, we have
\[P[\mathrm{binomial}(p,N)>tN]< \frac 12 (L(p,n)-p)\]
as long as $p\in [0,\delta^\prime]$. 

Use Lemma \ref{l:wideparenL<L} to further increase $N_0$ so that for all $N\geq N_0$ and $p\in [\min\{\delta,\delta^\prime\}, 1/2]$,
\[\wideparen{L}(p,N)\leq L(p,n).\]
For $p$ in the interval $[\min\{\delta,\delta^\prime\}, 1/2]$ and $N$ sufficiently large, we have
\begin{align*}
& P[\mathrm{binomial}(p,N)>tN]\cdot \wideparen{L}(p,N)+P[\mathrm{binomial}(p,N)\leq tN]\cdot L(p,N)\\
& \leq
P[\mathrm{binomial}(p,N)>tN]\cdot L(p,n)+(1-P[\mathrm{binomial}(p,N)>tN])\cdot L(p,n) \\
&= L(p,n),
\end{align*}
and if $p\leq \min\{\delta,\delta^\prime\}$,
\begin{align*}
& P[\mathrm{binomial}(p,N)>tN]\cdot \wideparen{L}(p,N)+P[\mathrm{binomial}(p,N)\leq tN]\cdot L(p,N)\\
& \leq P[\mathrm{binomial}(p,N)>tN] + L(p,N) \\
&\leq \frac 12 (L(p,n)-p) + \frac 12 (L(p,n)+p) \\
& = L(p,n).
\end{align*}
\end{proof}

\begin{lemma}
Given $n$ odd and $t\in (0,1/2)$, for all $N=N(n,t)$ (odd) large enough,
\begin{align*}
P[\mathrm{binomial}(p,N)&\in (tN,(1-t)N) ]\cdot \wideparen{L}(p,N)+ \\
& P[\mathrm{binomial}(p,N)\notin (tN,(1-t)N) ]\cdot L(p,N)
\\ & \leq L(p,n),
\end{align*}
over all $p\in [0,1]$.
\label{l:key}
\end{lemma}

\begin{proof}
Let $N=N(n,t)$ be chosen as in Lemma \ref{l:key0}. Write the expression on the left hand side above as
\begin{align*}
& P[\mathrm{binomial}(p,N)\in (tN,(1-t)N) ]\cdot \wideparen{L}(p,N)+ \\
& P[\mathrm{binomial}(p,N)\leq tN]\cdot L(p,N) + P[\mathrm{binomial}(p,N)\geq (1-t)N]\cdot L(p,N).
\end{align*}
For $p\in[0,1/2]$, bounding the third term by $P[\mathrm{binomial}(p,N)\geq (1-t)N]\cdot \wideparen L(p,N)$, we get the expression in Lemma \ref{l:key0}. For $p\in [1/2,1]$, we apply the same bound to the second term, and use the symmetry of the binomial distribution and the functions $L(p,n)$ and $\wideparen L(p,n)$:
\begin{align*}
&\leq P[\mathrm{binomial}(p,N)< (1-t)N ]\cdot \wideparen{L}(p,N)+ P[\mathrm{binomial}(p,N)\geq (1-t)N]\cdot L(p,N)\\
&=P[\mathrm{binomial}(1-p,N)> tN ]\cdot \wideparen{L}(1-p,N)+ P[\mathrm{binomial}(1-p,N)\leq tN]\cdot L(1-p,N),
\end{align*}
again applying Lemma \ref{l:key0}.
\end{proof}

\section{Partitioning rules}
\label{s:partitioning}

A partition, ${\mathscr P}$, of the domain (a standard Borel space) $\Omega$ is a finite family of disjoint measurable subsets, called cells, covering $\Omega$. To a partition ${\mathscr P}$ and a labelled sample $\sigma$ associate a classifier, ${h}_{\mathscr P}$, as follows. The predicted label of a point $x$ is determined by the majority vote among the elements of a labelled sample contained in the same cell as $x$. To avoid voting ties, we will remove if necessary the datapoint having the largest index, leaving an odd number of labels for the vote. The labels of those cells entirely missed by $\sigma$ are not relevant, and for instance can be chosen at random, or always be equal to $1$. (In our future rule, this will almost surely never happen.)

\begin{lemma}
Let ${\mathscr P}$ be a partition of the domain. Denote $p=P[Y=1]$.
Then, conditionally on each cell of the partition containing at least $n$ sample points, the expected error of the histogram classifier satisfies
\[\E L(h_{\mathscr P})\leq \wideparen L(p,n).\]
\label{l:err_partition}
\end{lemma}

\begin{proof}
Denote $p_C=P[Y=1\mid X\in C]$. Then $p=\sum \mu(C)p_C$.
Using the monotonicity of the function $L(p,n)$ in $n$ (Lemma \ref{l:monotone}),
\begin{align*}
P[h_{\mathscr P}(X)\neq Y\colon \sharp\sigma\upharpoonright C\geq n,~C\in P] &=
\sum_{C\in P}\mu(C)P[h_{\mathscr P}(X)\neq Y\mid X\in C,\sharp\sigma\upharpoonright C\geq n] \\
&\leq \sum_{C\in P}\mu(C)P[h_{\mathscr P}(X)\neq Y\mid X\in C,\sharp\sigma\upharpoonright C=n]
\\
&= \sum_{C\in P}\mu(C)L(p_C,n) \\
&\leq \wideparen L(p,n).
\end{align*}
\end{proof}

A partitioning rule $h=(h_{{\mathscr P}_n})$ is based on a sequence of partitions of the domain, $({\mathscr P}_n)$.  
Those partitions can be either deterministic and fixed in advance (as the histogram rule), or random, for instance determined by the (unlabelled) elements of a subsample.
To talk about random partitions, one needs of course a standard Borel structure  on the family of partitions that may emerge. This happens naturally, for example, in our case, where the partitions are into cyclic intervals of the circle: the family of all such partitions is naturally identifiable with a standard Borel space.

There are various known sufficient conditions for a partitioning rule to be  consistent. For example (\citet{DGL}, Th. 6.1) this is the case if $\Omega$ is a Euclidean domain, and the cell $C(X)$ containing a random element $X\in\Omega$ has two properties: the diameter of $C(X)$ converges to zero in probability, and the number of points of a sample contained in $C(X)$ converges to infinity in probability.

For a labelled sample $\sigma=(x_1,\ldots,x_n,y_1,\ldots,y_n)$, we denote $P_{\sigma}$ the corresponding empirical probability. In particular,
\[P_{\sigma}[Y=1] = \frac 1 n \sharp\{i\colon y_i=1\}.\]

The following lemma is our entire learning rule in a nutshell. It demonstrates the protocol for partitioning cells without increasing the error of the partitioning rule.

\begin{lemma}
Let the domain $\Omega$ be equipped with a learning problem $(\mu,\eta)$.
Let $\mathscr P$ be a random finite partition of $\Omega$, and $\sigma,\varsigma,\tau$ three jointly independent i.i.d. random labelled samples. Suppose also that $\mathscr P$ and $\varsigma$ are independent. Denote $n$ the size of $\sigma$ and $N$ the size of $\varsigma$.
Let $0<\e<1/2$, and let $N(n,\e)$ be chosen as in Lemma \ref{l:key}. Suppose  $N\geq N(n,\e)$.
Define a random partition $\mathscr Q$ as follows: if $P_{\varsigma}[Y=1]\in (\e,1-\e)$, then ${\mathscr Q}={\mathscr P}$, otherwise ${\mathscr Q}=\{\Omega\}$. 
Conditionally on the event that every cell of $\mathscr P$ contains at least $N$ points of $\tau$, 
\begin{align*}
\E L(h_{\mathscr Q}(\tau))\leq \E L(h_{\{\Omega\}}(\sigma)).
\end{align*}
\label{l:key_piece}
\end{lemma}

\begin{proof}
Denote for short the events
\[A=[P_{\varsigma}[Y=1]\in (\e,1-\e)]\mbox{ and }B =[\mbox{ for all cells } C\in {\mathscr P},\sharp\tau\upharpoonright C\geq N].\]
Denoting $p=P[Y=1]=\E\eta$, we have
\[P(A)=P[\mbox{\small binomial}(p,N)\in (\e,1-\e)],\]
and since the events $A$ and $B$ are independent,
\begin{align*}
\E (L(h_{\mathscr Q}(\tau))\mid B) &= P(A)\E (L(h_{\mathscr P}(\tau))\mid B) + (1-P(A))\E (L(h_{\{\Omega\}}(\tau))\mid B) \\
\mbox{\small (Lemma \ref{l:err_partition})}&\leq P(A)\wideparen L(p,N)+ (1- P(A)) L(p,N)\\
\mbox{\small (Lemma \ref{l:key})} &\leq L(p,n) \\
&=L(h_{\{\Omega\}}(\sigma)).
\end{align*}
\end{proof}

For $x\in\Omega$, let $C(x)$ denote the cell of the partition ${\mathscr P}_n$ containing $x$, and $N(x)$ the number of elements of $\sigma$ belonging to the cell $C(x)$. The following is a variation on Theorem 6.1 in \citet{DGL}. 

\begin{theorem}
Let $(\mu,\eta)$ be a learning problem on a standard Borel space $\Omega$.
Let $({\mathscr P}_k)$ be a sequence of random partitions of $\Omega$, and let $(D_k)$ be a sequence of finite i.i.d. labelled samples. Suppose that $\E(\eta\mid {\mathscr P}_k)\to \eta$ in probability, and the number $N(X)$ of elements of $D_k$ in a random cell $C(X)\in {\mathscr P}_k$ goes to infinity in probability as $k\to\infty$. Then the expected error $\E h_{{\mathscr P}_k}(D_k)$ converges to $L^{\ast}=L^{\ast}(\mu,\eta)$ as $k\to\infty$.
\label{th:variation}
\end{theorem}

\begin{proof}
Denote 
\[\hat\eta_k(x)= \frac{1}{N(x)}\sum_{i\colon X_i\in C(x)}Y_i\]
the empirical regression function. According to Corollary 6.1 in \citet{DGL}, it is enough to show that $\E\left\vert\hat\eta_k(X)-\eta(X)\right\vert \to 0$. By the triangle inequality,
\begin{align*}
\E\left\vert\hat\eta_k(X)-\eta(X)\right\vert &\leq
\E\left\vert\hat\eta_k(X)-\E(\eta\mid {\mathscr P}_k)(X)\right\vert + 
\E\left\vert\E(\eta\mid {\mathscr P}_k)(X)-\eta(X)\right\vert.
\end{align*}
The first term converges to zero through conditioning on $N(X)$ and using the fact that $N(x)\hat\eta_k(X)$ is distributed as $\mbox{\small binomial}(N(x),\E(\eta\mid {\mathscr P}_k)(x))$, it is exactly the first part of the proof of Theorem 6.1 in \citet{DGL}. The convergence to zero of the second term is our assumption.
\end{proof}

\section{Cyclic orders}
\label{s:cyclic}

Recall again a basic theorem in descriptive set theory: every standard Borel space of uncountable cardinality is isomorphic to the unit interval with its usual Borel structure (see Th. 15.6 in \citet{kechris}). In particular, every such space is Borel isomorphic to the unit circle:
\[\s^1=\{e^{2\pi \mathbf{i} t}\colon t\in [0,1)\}\subseteq{\mathbb{C}}.\]
Thus, given an arbitrary domain $\Omega$ (a standard Borel space), we can fix a Borel isomorphism with the circle $\s^1$ and work directly with the circle from now on. 

This is the same thing as choosing on $\Omega$ a cyclic order with certain properties, and we will give a minimum of necessary definitions. A {\em cyclic order} on a set $X$ is a ternary relation, denoted $[x,y,z]$, satisfying the following properties:

\begin{enumerate}
\item Either $[x,y,z]$ or $[z,y,x]$, but not both.
\item $[x,y,z]$ implies $[y,z,x]$.
\item $[x,y,z]$ and $[y,u,z]$ implies $[x,u,z]$.
\end{enumerate}

A linearly ordered set $(X,\leq)$ supports a cyclic order given by
\[[x,y,z] \mbox{ if and only if } x<y<z\mbox{ or } y<z<x\mbox{ or }z<x<y.\]
The circle has a natural cyclic order, where $x<y<z$ whenever $y$ is between $x$ and $z$ when we traverse the arc from $x$ to $z$ in the counter-clockwise direction (although clockwise would do just as well). Here is a definition not requiring geometric notions: for any $t,s,w\in [0,1)$, $\left[ e^{2\pi \mathbf{i} t},e^{2\pi \mathbf{i} s},e^{2\pi \mathbf{i} w}\right]$ if and only if $[t,s,w]$, where the cyclic order on the interval is defined as above. (See \citet{S}, remark to Lemma 1.)  

Any two points $x,y$ of a cyclically ordered set define an open interval, $(x,y)$, consisting of all points $z$ with $[x,z,y]$. Similarly one defines other types of intervals. We will be interested in half-open intervals of the form $[x,y) = (x,y)\cup\{x\}$. A cyclic order on a standard Borel space $\Omega$ is Borel if the corresponding ternary relation is a Borel subset of $\Omega^3$, which in particular implies that every interval is a Borel set.

It is easy to verify that the Vapnik--Chervonenkis dimension of the family of all intervals (open, closed, and half-open) of a cyclically ordered set with at least 3 points is exactly 3. Indeed, every three-point set is shattered, while the axioms imply that a set of four points cannot be shattered.

Fixing any point $\xi$ of a cyclically ordered set $X$, we obtain a linear order $<_{\xi}$ on $X$, with $\xi$ as the smallest element, and for all other elements, $y<_{\xi}z$ if and only if $[\xi,y,z]$. Now the original cyclic order is exactly the cyclic order defined by the linear order $<_{\xi}$. 

A cyclic order is {\em dense} if for every $x,y$, $x\neq y$, there is $z$ with $[x,y,z]$. A cyclic order is {\em order-separable} if there is a countable subset meeting each non-empty open interval. Say that a cyclic order is {\em Dedekind complete} if every non-empty proper subset $C$ has the greatest lower bound with regard to the linear order $<_{\xi}$ for every $\xi\notin C$. 
It can be shown that a standard Borel space equipped with a Dedekind complete dense order-separable Borel order admits a Borel isomorphism with the circle $\s^1$ preserving the cyclic order. Thus, technically, we construct our learning rule by fixing a cyclic order on a domain having the above listed properties, but it is more convenient to work by directly identifying the domain with the circle $\s^1$ and its standard cyclic order.

A mapping $f\colon X\to Y$ between two cyclically ordered sets is {\em monotone} if for all $x,y,z\in X$, whenever $f(x),f(y),f(z)$ are all pairwise distinct, we have $[x,y,z]$ if and only if $[f(x),f(y),f(z)]$. This is equivalent to saying that for some (or any) $\xi\in X$, the mapping $f$ is monotone non-decreasing with regard to the linear orders $<_{\xi}$ on $X$ and $<_{f(\xi)}$ on $Y$. A monotone map between two linearly ordered sets is monotone in this sense (but the converse does not hold). One can also talk of monotone maps between a cyclically ordered set and linearly ordered set. The composition of two monotone maps is monotone.

Perhaps it would be helpful to mention that the exponential map $\R\to\s^1$ is monotone on any interval of unit length, but not on the entire real line: for instance, $0<0.5<1.25$, therefore $[0,0.5,1.25]$ with regard to the cyclic order on $\R$, but the corresponding images $e^0=1$, $e^{\pi\mathbf{i}}=-1$ and $e^{\pi\mathbf{i}/2}=\mathbf{i}$ satisfy $[1,\mathbf{i},-1]$, that is, $[1,-1,\mathbf{i}]$ does not hold. Similarly, the two-fold cover of $\s^1\to\s^1$, $x\mapsto x^2$, is not cyclically monotone. On the contrary, every orientation-preserving self-homeomorphism of $\s^1$ is.
It is further easily seen that every monotone map from the circle $\s^1$ to itself is Borel.

Say that $y$ is a {\em successor} of $x$ in a finite cyclically ordered set $\mathcal P$, if for all $z\in {\mathcal P}\setminus\{x,y\}$ one has $[x,y,z]$, that is, $x\neq y$ and $[x,z,y]$ does not happen. Clearly, the successor of a given element always exists, provided $\vert {\mathcal P}\vert\geq 2$, and is unique.
Let now $\mathcal P$ be a finite subset of a cyclically ordered set $X$. Then $\mathcal P$ defines a partition of $X$ into half-open intervals $[x,y)$, for all pairs $x,y\in {\mathcal P}$ where $y$ is the successor of $x$ in $\mathcal P$. We will denote this partition $\hat {\mathcal P}$. If $\vert {\mathcal P}\vert\leq 1$, then by definition the corresponding partition is trivial, $\hat {\mathcal P}=\{\Omega\}$. (If there is a single point, $x$, in $\mathcal P$, then one may say the only half-open interval contained in $\hat {\mathcal P}$ is $[x,x)=\Omega$.)

\begin{lemma}
Let $f\colon X\to Y$ be a surjective monotone map between two cyclically ordered sets, and let ${\mathcal P}\subseteq X$ be a finite subset. Then every half-open interval in the partition $\widehat{f({\mathcal P})}$ of $Y$ defined by $f({\mathcal P})$ is the image of some interval of the partition $\hat {\mathcal P}$ of $X$ defined by $\mathcal P$.
\label{l:intervals}
\end{lemma}

\begin{proof}
Let $x,y\in {\mathcal P}$, where $b=f(y)$ is the successor of $a=f(x)$ in $f({\mathcal P})$. Denote $x^\prime$ the maximal element in the finite set $f^{-1}(a)$ with regard to the linear order $<_y$. The interval $[x^\prime,y)$ contains no other elements of $f^{-1}(a)$. Now let $y^\prime$ be the minimal element in the finite set $f^{-1}(b)$ with regard to the linear order $<_x$. The interval $[x^\prime,y^\prime)\subseteq [x^\prime,y)$ still contains no elements of $f^{-1}(a)$ other than $x^\prime$, and no elements of $f^{-1}(b)$ other than $y^{\prime}$. Then $y^\prime$ is the successor of $x^\prime$ in $\mathcal P$: any element $w$ of $\mathcal P$ strictly between those two would have either satisied $[a,f(w),b]$ or coinside with $a$ or $b$, both of which are impossible. 

We claim that in this case, $f[x^\prime,y^\prime)=[a,b)$. Let $w\in (a,b)$, that is, $[a,w,b]$. Since $f$ is surjective, there is $z\in X$ with $f(z)=w$. Because of monotonicity of $f$, we must have $[x^\prime,z,y^\prime]$, that is, $z\in (x^\prime,y^\prime)$. We conclude.

The trivial case $f(P)=\emptyset={\mathcal P}$ is obvious. Finally, suppose $f({\mathcal P})$ only contains one element, $a$, that is, $f^{-1}(a)={\mathcal P}$. If $Y$ only contains one element other than $a$, just select any interval of $\hat {\mathcal P}$ containing a preimage of this element. Else, we claim that all of $X\setminus {\mathcal P}$ is contained in only one interval of $\hat {\mathcal P}$. Indeed, let $x,y\in X\setminus {\mathcal P}$ be such that $f(x)\neq f(y)$. If $x$ and $y$ belong to different intervals of $\hat {\mathcal P}$, there exist $z,w\in {\mathcal P}$ with $[x,z,y]$ and $[x,y,w]$. This implies the incompatible properties $[f(x),a,f(y)]$ and $[f(x),f(y),a]$. From here the statement easily follows.
\end{proof}

If $f\colon X\to Y$ is a measurable map between two standard Borel spaces and $\nu$ is a Borel probability measure on $X$, then the pushforward measure $\nu\circ f^{-1}$ on $Y$ (which is also a Borel probability measure) is defined by letting $\nu\circ f^{-1}(A)= \nu(f^{-1}(A))$ for every Borel subset $A\subseteq Y$.

\begin{lemma}
Given a Borel probability measure on the circle $\s^1$, there is a monotone (hence Borel) map $f\colon\s^1\to\s^1$ with $\nu\circ i^{-1}=\mu$, where $\nu$ is the Haar measure on the circle.
\label{inumu}
\end{lemma}

\begin{proof}
The map $j\colon\s^1\ni e^{2\pi\mathbf{i} t}\mapsto t\in[0,1)$ is a Borel isomorphism. 
The push-forward measure $\nu\circ j^{-1}$ is the Lebesgue measure on the unit interval, $\lambda$, so $j$ is an isomorphism between the Lebesgue probability spaces $(\s^1,\nu)$ and $([0,1),\lambda)$. 
Denote $\mu^\prime=\mu\circ j^{-1}$ the push-forward measure, and let $F$ be the corresponding distribution function, $F(t)=\mu^\prime(-\infty,t]$ ($=\mu^\prime[0,t]$ for $t\in [0,1]$). Let $i^\prime\colon [0,1)\ni\theta\mapsto \inf\{t\in [0,1]\colon F(t)\geq \theta\}\in [0,1)$. This is a monotone map with $\lambda\circ i^{\prime -1}=\mu^\prime$. Finally, define $i=j^{-1}\circ i^{\prime}\circ j$. This is the desired monotone map from $\s^1$ to itself that pushes forward $\nu$ to $\mu$.
\end{proof}

\begin{lemma}
Given $k$, $N$, and $\delta>0$, there exists $M=M(k,n,\delta)$ so large that for every Borel probability measure $\mu$ on the circle $\s^1$, if $k+M$ i.i.d. points following the law $\mu$ are chosen, then with confidence $1-\delta$ every interval of the circle partition $\hat {\mathcal P}$ generated by the random finite set ${\mathcal P}=\{X_1,X_2,\ldots,X_k\}$ contains at least $N$ points from among $X_{k+1},X_{k+2},\ldots,X_{k+M}$.
\label{l:M}
\end{lemma}

\begin{proof}
First, we prove the lemma for $\s^1$ with the Haar measure. Fix a sufficiently small $\e>0$. The probability of all the intervals of the circular partition made by ${\mathcal P}=\{x_1,x_2,\ldots,x_k\}$ to have arc length $\geq\e$ is
\begin{align*}
(1-2\e)(1-4\e)\cdot\ldots\cdot(1-(k-1)\e) & > 1-2\e-4\e-\ldots- (k-1)\e \\
&=1-k(k-1)\e.
\end{align*}
Thus, if we set $\e=\delta/2k(k-1)$, then with confidence $1-\delta/2$ every interval will have length $\geq\e$. 

Since the VC dimension of the family of all half-open intervals of the circle is $d=3$, the sample size that suffices to empirically estimate the measure of all the intervals with confidence $1-\delta/2$ to within the precision $\e/2$ does not exceed
\[M^\prime = \max\left\{\frac{48}{\e}\log\frac{16e}{\e},\frac{8}{\e}\log\frac{4}{\delta} \right\}.\]
(Here we use the bounds from \citet{vidyasagar}, p. 269, Th. 7.8.) Set \[M=\max\left\{M^\prime,\frac{2N}{\e}\right\}.\]
For $n\geq M$, if $\sigma$ is an $n$-sample, then, denoting $\nu_n$ the empirical measure, we have with confidence $1-\delta$ that for each interval $I$ of the partition:
\[\nu_n(I)\geq \nu(I)-\frac{\e}2 \geq \frac {\e}2,\]
that is, $I$ contains at least $n\e/2\geq N$ points of the sample.

Now let $\mu$ be an arbitrary measure on $\s^1$. Select a monotone map $i\colon\s^1\to\s^1$ pushing forward the Haar measure $\nu$ to $\mu$ (Lemma \ref{inumu}). The random elements $X_1,\ldots,X_{n+k} \sim\mu$ can be written as $i(X^\prime_1),\ldots,i(X^\prime_{n+k})$, where $X^\prime_i$ are i.i.d. random elements following the law $\nu$. 
According to Lemma \ref{l:intervals}, for every interval of the partition generated by $X_1,\ldots,X_k$ its intersection with $i(\s^1)$ is the image of some interval of the partition generated by $X^\prime_1,\ldots,X^\prime_k$, and so, according to the first part of our proof, with confidence $1-\delta$, all those intervals contain at least $N$ sample points each.
\end{proof}

Say that a finite subset $\mathcal P$ of the circle $\s^1$ is $\e$-dense with regard to a probability measure $\mu$, if $\mathcal P$ meets every half-open interval of measure $\geq \e$. 

\begin{lemma}
Let $\mu$ be a Borel probability measure on the circle $\s^1$, and let $X_1,X_2,\ldots$ be a sequence of i.i.d. random elements of $\s^1$ following the law $\mu$. Let $\e>0$. Almost surely, starting with some $k$ large enough, the random finite set $\{X_1,\ldots,X_k\}$ is $\e$-dense.
\label{l:edense}
\end{lemma}

\begin{proof}
Fix a cyclically monotone parametrization $i\colon\s^1\to\s^1$ pushing forward the Haar measure $\nu$ to $\mu$ (Lemma \ref{inumu}). 
Let $\mathcal Q$ be a cover of the circle with $n_0\geq 2\e^{-1}$ intervals of Haar measure between $\e/3$ and $\e/2$ each. Let $Y_1,\ldots,Y_k$ be i.i.d. random elements of $\s^1$ following the law $\nu$.
The probability for all of them to miss at least one of the intervals from $\mathcal Q$ is bounded by $n_0(1-\e/3)^k$, and this is a summable sequence in $k$. By the Borel-Cantelli lemma, almost surely, starting with some $k$ high enough, in every interval $I\in {\mathcal Q}$ there is contained at least one random element from among $Y_i$, $i=1,2,\ldots,k$. 
Let $J$ be a cyclic interval with $\mu(J)\geq \e$. The inverse image $i^{-1}(J)$ is again a cyclic interval by the definition of a monotone map, and $\nu(i^{-1}(J))=\mu(I)$. The interval $i^{-1}(J)$ must wholly contain at least one interval $I\in {\mathcal Q}$. We conclude: almost surely, some $X_i=i(Y_i)$ belongs to $J$.
\end{proof}

\section{The learning rule}
\label{s:rule}

Select a sequence $(\e_n)$ of positive numbers converging to zero, with $\e_1<1/2$. Select a summable sequence of positive numbers $(\delta_n)$, that is, $\sum_{n=1}^{\infty}\delta_n<\infty$, satisfying $\delta_1<1$.

Put $a_1=0$, $b_1=1$, and further select $N_k,a_k,b_k$, $k>1$, recursively as follows.

\begin{enumerate}
\item Let $N_k=N(b_{k-1},\e_k)$ be chosen as in Lemma \ref{l:key}, with $n=b_{k-1}$ and $t=\e_k$. 

In other words, for all $N\geq N_k$, $N$ odd, and all $p\in [0,1]$,
\begin{align*}
P[\mathrm{binomial}(p,N)&\in(\e_kN,(1-\e_k)N)]\cdot \wideparen{L}(p,N)+ \\
& P[\mathrm{binomial}(p,N)\notin (\e_kN,(1-\e_k)N)]\cdot L(p,N) \leq L(p,b_{k-1}).
\end{align*}

\item Choose $a_k=M(k,N_k,\delta_k)$ as in Lemma \ref{l:M}.

That is, $a_k$ is so large that for every Borel probability measure $\mu$ on the circle $\s^1$, if $k+a_k$ i.i.d. points $\sim\mu$ are chosen, then with confidence $1-\delta_k$ every interval of the circle partition generated by the random finite set ${\mathcal P}=\{X_1,X_2,\ldots,X_k\}$ contains at least $N_k$ elements from among $X_{k+1},X_{k+2},\ldots,X_{k+a_k}$.

\item Now choose $b_k$, again using Lemma \ref{l:M}, as $b_k=M(k,a_k,\delta_k)$. 

In full, for every Borel probability measure $\mu$ on $\s^1$, if $k+b_k$ i.i.d. points $\sim\mu$ are chosen, then with confidence $1-\delta_k$ every interval of the partition generated by $\{X_1,X_2,\ldots,X_k\}$ contains at least $a_k$ elements from among $X_{k+1},\ldots,X_{k+b_k}$.
\end{enumerate}

Set $n_1=1$ and further, recursively,
\[n_{k}=n_{k-1}+a_k+b_k+1.\]

Denote $A_1=\emptyset$, $B_1=\{1\}$, and for $k>1$,
\[A_k=(n_{k-1}+2,\ldots,n_{k-1}+a_k+1),~~
B_k=(n_{k-1}+a_k+2,\ldots,n_{k-1}+a_k+b_k+1).\]
Denote ${\mathcal P}_1=\emptyset$ and for every $i\geq 2$ set 
\[{\mathcal P}_i=\{x_{n_j+1}\colon j=1,\ldots,i-1\}.\] 

For a finite subset $I$ of the positive integers and a labelled sample $\sigma$, we will denote $\sigma[I]$ a labelled subsample of $\sigma$ consisting of all pairs labelled with $i\in I$, in the same order.

Recall further that for a finite set $\mathcal Q$, we denote $\hat{\mathcal Q}$ the partition of the circle $\s^1$ into half-open cyclic intervals determined by the finite set ${\mathcal Q}$. Also, given a partition $\mathscr P$, the corresponding histogram classifier is denoted $h_{\mathscr P}$. 

Finally, $P_{\sigma[A_i]}$ is the (conditional) empirical probability supported on the subsample $\sigma[A_i]$, in particular,
\[P_{\sigma[A_i]}[Y=1\vert X\in I] = \frac{\sharp\{j\in A_i\colon x_j\in I,~y_j=1\}}{\sharp\{j\in A_i\colon x_j\in I\}}.\]

Here is the algorithm description.
\newpage

\begin{quote}
\hrule
\vskip .2cm
\begin{tabbing}
\quad \=\quad \=\quad \=\quad\=\quad\=\quad\=\quad\=\quad\=\quad\kill
\keyw{on input} $\sigma_n$ \keyw{do} \\
\> $k\leftarrow\max\{i\colon n_i\leq n\}$ \\
\> ${\mathcal Q}\leftarrow\emptyset$ \\
\> ${\mathcal R}\leftarrow\emptyset$ \\
\> \keyw{for} $i=1:k$ \keyw{do} \\
\>\> \keyw{if} every interval $I \in\hat{\mathcal P}_{i}$ contains $\geq a_i$ points of $\sigma[B_i]$ \keyw{and}\\
\>\>\>  ($i=1$ \keyw{or} every interval $I \in \hat {\mathcal Q}$ contains $\geq N_i$ points of $\sigma[A_i]$) \keyw{do} \\
\>\>\>\> \keyw{if} $k>1$ \keyw{do} \\
\>\>\>\>\> \keyw{for} every  $I \in\hat{\mathcal Q}$  \keyw{do} \\
\>\>\>\>\>\> \keyw{if} 
$P_{\sigma[A_i]}[Y=1\vert X\in I]\in (\e_i,1-\e_i)$, 
\keyw{do} \\
\>\>\>\>\>\>\> ${\mathcal R}\leftarrow {\mathcal R}\cup ({\mathcal P}_i\cap I)$ \\ 
\>\>\>\>\>\>\> \keyw{end do}\\
\>\>\>\>\>\>  \keyw{end if} \\
\>\>\>\>\>\> \keyw{end do}\\
\>\>\>\>\> \keyw{end for}\\
\>\>\>\>  \keyw{end if} \\
\>\>\>\> ${\mathcal Q}\leftarrow {\mathcal R}$ \\
\>\>\>\> $H\leftarrow {h}_{\hat{\mathcal Q}}(\sigma[B_i])$ \\
\>\>\> \keyw{end do} \\
\>\> \keyw{end if} \\
\>\keyw{end for} \\
\keyw{end do} \\
\keyw{return} $H$
\end{tabbing} 
\vskip .2cm

\hrule
\vskip .4cm
\end{quote}

\section{Monotonicity of the expected error} 
\label{s:monotonicity}

The hypothesis can only be updated at the moments $n=n_k$, so 
it is enough to compare the expected error of $g_{n_{k-1}}$ and $g_{n_{k}}$. Denote $i$ the largest integer $<k$ such that the hypothesis was updated at the step $n_i$. Denote ${\mathcal Q}_i$ the state of the partitioning set $\mathcal Q$ at the moment $n=n_i$. This is a random finite subset of the circle with $\leq i$ elements. As before, we denote $\hat {\mathcal Q}_i$ the family of half-open intervals into which the circle is partitioned by the finite set ${\mathcal Q}_i$. We will be conditioning on $k$, $i$, and ${\mathcal Q}_i$, so from now on, the integers $i,k$ and a finite subset ${\mathcal Q}_i\subseteq\s^1$ (possibly empty) are fixed, while ${\mathcal Q}_k\supseteq {\mathcal Q}_i$ stays random, and we do not know whether a hypothesis update was made at the time $n_k$. We will further condition on the event (A) ``every interval of $\hat {\mathcal Q}_k$ contains at least $N_i$ points of the testing sample $\sigma[A_k]$'', because given the complementary event, no testing and update were made and $h_{n_k}=h_{n_i}=k_{n_{k-1}}$.

It is now enough to verify, conditionally on the above, that for every interval $I\in\hat {\mathcal Q}_i$,
\begin{equation}
P[g_k(X)\neq Y\mid X\in I]\leq P[g_i(X)\neq Y\mid X\in I].
\label{eq:I}
\end{equation}
Fix such an interval $I$. Conditioning further on the size of the samples $\sigma[B_i]\upharpoonright I$, $\sigma[A_k]\upharpoonright I$, and $\sigma[B_k]\upharpoonright I$, we see they are conditionally i.i.d., and conditionally jointly independent. The sample $\sigma[A_k]\upharpoonright I$ is conditionally independent on the random partition ${\mathcal P}_k\cap I$. Moreover, conditionally on the event (A) above, we have $m_k=\sharp\sigma[A_k]\upharpoonright I\geq N(b_{k-1},\e_k)$, where $b_{k-1}>\sharp\sigma[B_i]\upharpoonright I$. 
We are under the assumptions of Lemma \ref{l:key_piece}. 

Denote $\hat{\mathcal P}_k[I]$ the family of all the intervals of the partition $\hat{\mathcal P}_k$  contained in $I$. This is a finite random partition of $I$ (possibly trivial), given by the random set ${\mathcal P}_k\cap I$. For every interval $J\in \hat{\mathcal P}_k[I]$, set $m_J=\sharp\sigma[B_k]\upharpoonright I$. According to Lemma \ref{l:key_piece}, conditionally on the event ``for all $J$, $m_J\geq a_k$'' the inequality (\ref{eq:I}) above holds. Since it also holds trivially conditionally on the complementary event (in which case it turns into equality), we are done.

\section{Universal consistency} 
\label{s:consistency}

The difficulty here is that the diameter of a random cell (that is, an interval $I=I(X)$ in $\hat{\mathcal Q}_k$ containing a random element $X$) need not converge to zero in probability, and not only because of $\eta$. Enough to consider the case where the measure $\mu$ is supported on an atom located at $1$ and a small arc of length $\e>0$ around $-1$. Almost surely, starting with some $k$, $\hat{\mathcal Q}_k$ will contain two intervals of arc length $>1/2-\e$ each.

Analysis of the proof of Theorem 6.1 in \citet{DGL} shows that the requirement of the cell diameter going to zero in probability is only needed in order to prove that the sequence of conditional expectations of the regression function $\eta$ formed with regard to the sequence of random partitions converges to $\eta$. This would be, in our case,
\begin{equation}
\E(\eta\mid \hat {\mathcal Q}_k)\overset p\to \eta.
\label{l:etamidQk}
\end{equation}
We will prove it directly.

\begin{lemma}
Let $(\mu,\eta)$ be a learning problem on the circle $\s^1$. Almost surely, starting with some $k$ large enough, at every step $n_k$ every interval of the random partition $\hat{\mathcal Q}_k$ will be tested and the hypothesis will be updated.
\label{l:as}
\end{lemma}

\begin{proof}
By the choice of $a_k$, the event ``every interval of the random partition $\hat {\mathcal P}_k$ contains more than $N=N(\e_k,b_{k-1})$ points of $\sigma[A_k]$'' occurs with probability $>1-\delta_k$, and by the choice of $b_k$, the event ``every interval of the random partition $\hat{\mathcal P}_k$ contains more than $a_k$ points of $\sigma[B_k]$'' occurs with probability $>1-\delta_k$ as well. Since $(\delta_k)$ is a summable sequence, we conclude.
\end{proof}

\begin{lemma}
Let $(\mu,\eta)$ be a learning problem on the circle $\s^1$. Let $\e>0$. Almost surely, starting with some $k$ large enough, for every interval $I$ of the random partition $\hat {\mathcal Q}_k$ having the property $p=P[Y=1\mid X\in I]\in (\e,1-\e)$ we will have ${\mathcal P}_{k+1}\cap I \subseteq {\mathcal Q}_{k+1}$.
\label{l:addset}
\end{lemma}

\begin{proof}
From Lemma \ref{l:as}, we know that almost surely, for all $k$ large enough, the cells of the partition ${\mathcal P}_k$ will be tested. For $k^\prime$ sufficiently large, $\e_{k^\prime}<\e/2$.  According to the Chernoff bound, 
\begin{align*}
P[\mbox{\small binomial}(a_{k^\prime},p)\notin[\e_{k^\prime},1-\e_{k^\prime}]] 
&\leq
P[\left\vert\mbox{\small binomial}(a_{k^\prime},p)-p\right\vert >\e/2] \\
<e^{-\e^2 a_{k^\prime}/4}.
\end{align*}
The series is summable, and by the Borel--Cantelli lemma, we conclude that the divisibility of $I$ will be certified almost surely from some step on. Consequently, our algorithm prescribes to add the set ${\mathcal P}_{k+1}\cap I$ to the partition ${\mathcal Q}_k$.
\end{proof}

\begin{lemma}
Let $(\mu,\eta)$ be a learning problem on the circle $\s^1$. Let $I$ be a half-open cyclic interval on which $\eta$ is neither a.e. equal to $1$ nor a.e. equal to $0$.
Almost surely, at some step $k$ we will have ${\mathcal Q}_k\cap I\neq\emptyset$.
\label{l:nonpure}
\end{lemma}

\begin{proof}
We have $p=P[Y=1\mid X\in I]\in (0,1)$. 
Every interval $J$ containing $I$ satisfies $P[Y=1\mid X\in J]\in (p\mu(I),1-p\mu(I))$. Almost surely, if $k$ is large enough, ${\mathcal P}_{k+1}\cap I\neq\emptyset$ (Lemma \ref{l:edense}), and either ${\mathcal Q}_k\cap I\neq\emptyset$, or else the interval $J_k$ of the partition $\hat {\mathcal Q}_k$ containing $I$ will be tested at the step $k+1$ and the set ${\mathcal P}_{k+1}\cap J_k$ added to the partitioning set (Lemma \ref{l:addset}). Thus, almost surely, ${\mathcal Q}_{k+1}\cap I\neq\emptyset$.
\end{proof}

Denote $\Sigma(\cup_k\hat {\mathcal Q}_k)$ the sigma-algebra generated by all the cyclic intervals determined by random partitions $({\mathcal Q}_k)$, $k\in\N$. Turns out, this random sigma-algebra 
has a rather transparent structure. We will clarify it now, as well as show that $\Sigma(\cup_k\hat {\mathcal Q}_k)$ is a bona fide random variable taking values in a standard Borel space.

Given a subset $A\subseteq\s^1$, denote $\Sigma_A$ the sigma-algebra on the circle generated by all cyclic intervals $[a,b)$, $a,b\in A$. It is a sub-sigma-algebra of the Borel algebra. 

\begin{lemma}
A subset $A\subseteq\s^1$ and its closure, $\bar A$, generate the same sigma-algebra.
\label{l:AbarA}
\end{lemma}

\begin{proof}
The inclusion $\Sigma_A\subseteq\Sigma_{\bar A}$ is trivial. Now suppose $a\in A$ and $b\in \bar A$. If there is a sequence of elements of $A$ with $b_n\uparrow b$ (that is, $[a,b_n,b]$), then $[a,b)=\cup_n[a,b_n)$. If there is a sequence $b_n\downarrow b$ ($[a,b,b_n]$), then $\{b\}=\cap_n[b,b_n)$, and $[a,b)=\cap_n [a,b_n]\setminus\{b\}$. Assume now $a,b\in A$ arbitrary. If there is a sequence  of elements of $A$, $a_n\uparrow a$, then $[a,b)=\cap_n[a_n,b)$; if there is a sequence $a_n\downarrow a$, then $\{a\}=\cap_n [a,a_n)$ and so on.
\end{proof}

\begin{lemma}
On every closed subset $F$ of $\s^1$ the sigma-algebra $\Sigma_F$ induces the standard Borel structure (as induced from $\s^1$).
\label{l:induceBorel}
\end{lemma}

\begin{proof}
Enough to show that for every $a,b\in F$, $a\neq b$, we have $(a,b)\cap F\in\Sigma_F\vert_F$. If $(a,b)\cap F=\emptyset$, it is clear; assume the contrary. There is a sequence $(a_n)$ of elements of $F$ with $a_n\in (a,b)$ and $a_n\downarrow \inf_{<_a}(a,b)\cap F$. We have $\{a\}=\cap_n[a,a_n)\cap F\in\Sigma_F\vert_F$, and finally $(a,b)\cap F=[a,b)\cap F\setminus\{a\}\in\Sigma_F\vert_F$.
\end{proof}

It is well-known and easily proved that every open subset $U$ of the real line (hence, of the circle) is uniquely represented as a union of disjoint open intervals (its connected components) whose endpoints belong to the complement of $U$, see e.g. \citet{A}, \S 5, Th. 21, or \citet{E}, Exercise 3.12.4(b).

\begin{lemma}
Let $F$ be a closed subset of the circle. Suppose the sigma-algebra $\Sigma_F$ is non-trivial (equivalently, $F$ contains at least two points).
Those atoms of $\Sigma_F$ that are not singletons are exactly the half-open intervals $[a,b)$ such that $(a,b)$ is a connected component of the complementary set $F^c=\s^1\setminus F$.
\label{l:atom}
\end{lemma}

\begin{proof}
Let $a,b\in F$, $a\neq b$. We have $[a,b)\in\Sigma_F$. Assume that $(a,b)\subseteq F^c$. The restriction $\Sigma_F\vert_{[a,b)}$ is generated, as a sigma algebra, by the intersections of the generating sets $[c,d)$, $c,d\in F$, with $[a,b)$. Since every such set either contains $[a,b)$ or is disjoint from it, the sigma-algebra  $\Sigma_F\vert_{[a,b)}$ is trivial. Altogether it means $[a,b)$ is an atom of $\Sigma_F$.

Let now $A\in\Sigma_F$ be an atom. Suppose it contains at least two points. For any two $a,b\in F$, $a\neq b$, exactly one of the intervals $[a,b)$ and $[b,a)$ contains $A$ as a subset. Denote $I$ the intersection of all the intervals $[a,b)$, $a,b\in F$ that contain $A$. Since $F$ is closed, the endpoints $c,d$ of the interval $I$ belong to $F$. As $A$ is an atom, it must satisfy $A\subseteq [c,d)$, and $(c,d)$ contains no points of $F$. Since $[c,d)$ is an atom by the first part of the proof, $A=[c,d)$.
\end{proof}

The map $F\mapsto \Sigma_F$ is not injective even on the closed subsets: for instance, all one-element subsets generate the same trivial sigma-algebra $\{\s^1\}$. 

\begin{lemma}
If $F\neq G$ are two distinct closed subsets and at least one of them contains two elements, then 
$\Sigma_F\neq\Sigma_G$. 
\end{lemma}

\begin{proof}
Suppose $F\supseteq\{a,b\}$, $a\neq b$. If $F\setminus G\neq\emptyset$, then for any $c\in F\setminus G$ we have $[a,c)\in\Sigma_F\setminus \Sigma_G$. So we can assume $F\subseteq G$. In this case, for any $d\in G\setminus F$, $[a,d)\in\Sigma_G\setminus\Sigma_F$.
\end{proof}

We can therefore bijectively identify the family of all sigma-algebras of the form $\Sigma_F$ with the family of all closed subsets of the circle with at least two elements, plus the trivial sigma-algebra $\{\s^1\}$. 

The family ${\mathcal F}(K)$ of closed subsets of a compact metric space is itself a compact metric space and therefore a standard Borel space, for example, when equipped with the Hausdorff distance (\citet{kechris}, 4.F.):
\[d(F,G)=\inf\{\e>0\colon d(x,F)<\e,d(y,G)<\e\mbox{ for all }x\in G,y\in F\}.\]
The subfamily of sets with at least two elements is open, hence Borel. The union of two standard Borel spaces is a standard Borel space. This gives a standard Borel structure to the family of all sigma-algebras of the form $\Sigma_F$, $F$ is a closed subset of $\s^1$. 

The sigma-algebras $\Sigma(\cup_k\hat {\mathcal Q}_k)$ that we are interested in are exactly of the form $\Sigma_{{\mathcal Q}_{\infty}}$, where we denote ${\mathcal Q}_{\infty}=\cup_k {\mathcal Q}_k$ the set of all partitioning points added by our algorithm. This inclusion $\Sigma(\cup_k\hat {\mathcal Q}_k)\subseteq \Sigma_{{\mathcal Q}_{\infty}}$ is clear, and if $a,b\in {\mathcal Q}_{\infty}$, then for some $k$, $a,b\in {\mathcal Q}_k$, and $[a,b)$ is in the sigma-algebra determined by the partition $\hat {\mathcal Q}_k$. 

Finally, the random variable with values in the above standard Borel space that we call a random sigma-algebra is realized through a map sending a sample path in $(\s^1\times\{0,1\})^{\infty}$ to the sigma-algebra ${\mathcal Q}_{\infty}$. This map is Borel measurable with regard to the above Borel structure. Indeed, it is a combination of the sequence of maps $(\s^1\times\{0,1\})^{[n_{k}+1,n_{k+1}]}$ to $(\s^1)^k$, produced by the learning rule, each of which can be expressed by a finite first-order formula with relation symbols $[\,,\,,\,]$ and $<$ and the real numbers as constants, and so is measurable, and the map sending a sequence $(x_k)$ to the closure of the set $\{x_k\}$. The measurability of the latter map can be seen as follows: the inverse image of the Hausdorff $\e$-neighbourhood of a closed set $F$ consists of all sequences satisfying the formula
\[\exists n \forall k, B_{1/n}(x_k)\subseteq F_\e,\]
making it a Borel set.

Here is a corollary of Lemma \ref{l:nonpure}.

\begin{lemma}
Either almost surely the sigma-algebra $\Sigma_{{\mathcal Q}_{\infty}}$ is trivial (and this is the case if and only if the regression function $\eta$ is constant a.e., taking value $0$ or $1$), or almost surely it is non-trivial. 
\label{l:trivialSigma}
\end{lemma}

\begin{lemma}
Almost surely, 
\begin{enumerate}
\item on the random closed set $F=\bar {\mathcal Q}_{\infty}$, the random sigma-algebra $\Sigma_{{\mathcal Q}_{\infty}}$ induces the standard Borel structure ${\mathcal B}_F$ coming from $\s^1$, and
\item the regression function $\eta$ assumes a.e. a constant value $0$ or $1$ on every atom of $\Sigma_{{\mathcal Q}_{\infty}}$ that is not a singleton.
\end{enumerate}
\label{l:etaatoms}
\end{lemma}

\begin{proof}
The first claim follows from Lemmas \ref{l:AbarA} and \ref{l:induceBorel}.

For the second claim, according to Lemma \ref{l:trivialSigma}, it is enough to consider the case where $\Sigma_{{\mathcal Q}_{\infty}}$ is almost surely non-trivial. 
It follows from Lemma \ref{l:nonpure} that almost surely, every interval with rational endpoints on which $\eta$ does not take a.e. identical value $0$ or $1$ will be divided at the $k$-th step for some $k$ large enough. We conclude that, almost surely, on every interval with rational endpoints contained in some atom of $\Sigma_{{\mathcal Q}_{\infty}}$ the regression function takes a.e. the identical value $0$ or the value $1$. It follows that almost surely, for every atom $A$ that is non-singleton and so has the form $A=[a,b)$ for $a,b\in \bar {\mathcal Q}_{\infty}$, on the corresponding open interval $(a,b)$ $\eta$ takes identical value $0$ or $1$ a.e. For those atoms with $\mu\{a\}=0$, the proof is over. 

Now denote $\mathcal U$ the family of all half-open intervals of the form $[a,b)$, where $\mu\{a\}>0$ and $b$ is rational. The family $U$ is countable, so again applying Lemma \ref{l:nonpure}, we conclude that almost surely, if any such interval is an atom, then $\eta$ must take the same value at the left endpoint $a$ as a.e. on the rest of the interval (this includes also the case $\mu(a,b)=0$).
\end{proof}

\begin{lemma}
Almost surely, $\E(\eta\mid \Sigma_{{\mathcal Q}_{\infty}})=\eta$.
\label{l:conditionalexpectation}
\end{lemma}

\begin{proof}
Select a Borel measurable version of $\eta$. Further, on every nontrivial atom of $\Sigma_{{\mathcal Q}_{\infty}}$ replace $\eta$ with a suitable constant value, either identically $0$ or identically $1$ (Lemma \ref{l:etaatoms},(2)). The union, $A$, of the countable family of nontrivial atoms belongs to our sigma-algebra, and the restriction of $\eta$ to $A$ is $\Sigma_{{\mathcal Q}_{\infty}}$-measurable. We have $A^c\subseteq \bar {\mathcal Q}_{\infty}$, therefore, almost surely the restriction of $\Sigma_{{\mathcal Q}_{\infty}}$ induces the standard Borel structure on $A^c$ (Lemma \ref{l:etaatoms},(1)) and the restriction of $\eta$ to $A^c$ is $\Sigma_{{\mathcal Q}_{\infty}}$-measurable as well. We conclude: our realization of $\eta$ is $\Sigma_{{\mathcal Q}_{\infty}}$-measurable.
\end{proof}

\begin{lemma} 
Almost surely, $\E(\eta\mid \hat {\mathcal Q}_k)\to \eta$.
\label{l:martingale_convergence}
\end{lemma}

\begin{proof}
Follows from the forward martingale convergence theorem (\citet{D}, Sect. IX.14) and Lemma \ref{l:conditionalexpectation}.
\end{proof}

And finally, the proof of the universal consistency of our learning rule, $g$. 

Denote $h$ the following variant of $g$: it is a partitioning rule based on the same sequence of random partitions $\hat {\mathcal Q}_k$ and labelling samples $\sigma[B_k]$, but updated at every moment $n_k$, irrespective of the number of sample points in the cells of the partition:
\[h_{n_k}(\sigma)=h_{\hat {\mathcal Q}_k}(\sigma[B_k]).\]
Lemma \ref{l:martingale_convergence} implies the almost sure convergence of the conditonal expectations $\E(\eta\mid \hat {\mathcal Q}_k)$ to $\eta$. Because of Lemma \ref{l:as} and the fact that $a_k\to\infty$, almost surely the smallest number of points of the i.i.d. sample $\sigma[B_k]$ contained in any cell of the random partition $\hat {\mathcal Q}_k$ at the step $n_k$ will go to infinity as $k\to\infty$. We are under the assumptions of Theorem \ref{th:variation}, and conclude that the rule $h$ is consistent.

The only difference between $g$ and $h$ is that $g$ sometimes delays the hypothesis update. More exactly, we have a certain random function, $F$, from the natural numbers to itself with the property $F(k)\leq k$ for all $k$, and the learning rule $g$ is defined from $h$ as follows:
\[g_{n_k}(\sigma) = h_{n_{F(k)}}(\sigma\upharpoonright [n_{F(k)}]).\]
Notice that $F(k)-k\to 0$ almost surely (Lemma \ref{l:as}). We are under the assumptions of Lemma \ref{l:random_delay} and conclude that the rule $g$ is consistent.

\section{Concluding remarks}
\label{s:remarks}

I am grateful to the two anonymous referees whose comments have helped to improve the readability of the paper. 

In connection with the discussion at the start of Sect. \ref{s:consistency},
it was pointed out by one referee that there are indeed examples of consistent partition-based algorithms without the diameter of the largest cell converging to zero in probability (\citet{SBV}). 

A Borel isomorphism between an Euclidean domain and the circle (Sect. \ref{s:cyclic}) is indeed not easy to implement algorithmically. However, already a Borel injection would suffice, and this can be coded in a constructive way, cf. \citet{P}, Sect. 7. Still, the learning rule described in the present article will be too slow for practical applications: its algorithmic efficiency is admittedly very low.
It remains an interesting challenge, to find a ``natural'' learning algorithm having the monotone expected learning error.

\vskip 0.2in
\bibliography{smart_jmlr_revised}

\end{document}